\pdfoutput=1

\documentclass[11pt]{article}

\usepackage{acl}

\usepackage{times}
\usepackage{latexsym}

\usepackage[T1]{fontenc}

\usepackage[utf8]{inputenc}

\usepackage{microtype}

\usepackage{enumitem}
\usepackage{tikz}
\newcommand*\circled[1]{\tikz[baseline=(char.base)]{
            \node[shape=circle,draw,inner sep=2pt] (char) {{\footnotesize#1}};}}

\usepackage{linguex}

\usepackage{hyperref}
\usepackage{xurl}
\usepackage{caption}
\usepackage{subcaption}
\usepackage{amsmath}
\usepackage{inconsolata}
\usepackage{amssymb}
\usepackage{amsfonts}
\usepackage{amsthm}
\usepackage{xfrac}
\usepackage{bbm}
\usepackage{amsthm}
\usepackage{graphicx}
\usepackage{combelow}
\usepackage{qtree}
\usepackage{CJKutf8}
\usepackage{subcaption}
\usepackage{array}
\usepackage{tagpdf}
\usepackage{accsupp}
\usepackage{multicol}

\usepackage{booktabs}
\usepackage{cleveref}
\newtheorem{defn}{Definition}
\newtheorem{example}{Example}
\newtheorem{thm}{Theorem}

\crefname{section}{\S}{\S\S}
\Crefname{section}{\S}{\S\S}
\crefname{table}{Table}{}
\crefname{figure}{Figure}{}
\crefname{algorithm}{Algorithm}{}
\crefname{equation}{eq.}{eqs.}
\crefname{appendix}{App.}{}
\crefname{prop}{Prop.}{}
\crefformat{section}{\S#2#1#3}  
\newcommand{\citeposs}[1]{\citeauthor{#1}'s (\citeyear{#1})}

\usepackage{todonotes}

\makeatother

\newcommand{\hleftenc}{\overrightarrow{\boldsymbol h}^{(\mathrm{enc})}}
\newcommand{\hleftdec}{\overrightarrow{\boldsymbol h}^{(\mathrm{dec})}}
\newcommand{\hrightenc}{\overleftarrow{{\boldsymbol h}}^{(\mathrm{enc})}}
\newcommand{\henc}{{\boldsymbol h}^{(\mathrm{enc})}}
\newcommand{\xx}{\mathbf{x}}
\newcommand{\yy}{\mathbf{y}}
\newcommand{\Aa}{\mathbf{a}}

\newcommand{\calA}{\mathcal{A}}
\newcommand{\func}{\xi} 

\newcommand{\jumpin}{\texttt{jump}}
\newcommand{\walkin}{\texttt{walk}}
\newcommand{\runin}{\texttt{run}}
\newcommand{\lookin}{\texttt{look}}
\newcommand{\twicein}{\texttt{twice}}
\newcommand{\leftin}{\texttt{left}}
\newcommand{\rightin}{\texttt{right}}
\newcommand{\aroundin}{\texttt{around}}
\newcommand{\andin}{\texttt{and}}
\newcommand{\thricein}{\texttt{thrice}}
\newcommand{\afterin}{\texttt{after}}
\newcommand{\jumpout}{\textsc{jump}}
\newcommand{\walkout}{\textsc{walk}}
\newcommand{\runout}{\textsc{run}}
\newcommand{\lookout}{\textsc{look}}
\newcommand{\leftout}{\textsc{lturn}}
\newcommand{\rightout}{\textsc{rturn}}

\newcommand{\SCAN}{{\small \textsf{SCAN}}}

\newcommand{\gdecode}{G\operatorname{-Dec}}
\newcommand{\gconv}{G\operatorname{-Conv}}
\newcommand{\gembed}{G\operatorname{-Embed}}
\newcommand{\dom}{\mathrm{dom}}

\newcommand{\vpsi}{\boldsymbol{\psi}}
\newcommand{\vphi}{\boldsymbol{\phi}}
\newcommand{\vf}{\boldsymbol{f}}
\newcommand{\R}{\mathbb{R}}

\newcommand{\bigo}[1]{\mathcal{O}\left(#1\right)}

\newcommand{\oldversion}[1]{}

\DeclareMathOperator*{\argmax}{argmax}

\usepackage{relsize}

\everypar{\looseness=-1}

\title{Equivariant Transduction through Invariant Alignment}

\usepackage{tipa}

\newcommand{\ucambridge}{\normalfont \text{\textipa{D}}}
\newcommand{\ethz}{\text{\normalfont \textipa{Q}}}

\author{Jennifer White$^{\ucambridge}$~\;~~\;~Ryan Cotterell$^{\ethz}$ \\
  $^{\ucambridge}$University of Cambridge~\;~$^{\ethz}$ETH Z\"{u}rich \\
  \texttt{\href{mailto:jw2088@cam.ac.uk}{jw2088@cam.ac.uk}}~\;~\texttt{\href{mailto:ryan.cotterell@inf.ethz.ch}{ryan.cotterell@inf.ethz.ch}}
}

\makeatletter
\hypersetup{pdftitle={\@title}, pdflang=en-us}
\makeatother
\begin{document}
\maketitle
\begin{abstract}
The ability to generalize compositionally is key to understanding the potentially infinite number of sentences that can be constructed in a human language from only a finite number of words.
Investigating whether NLP models possess this ability has been a topic of interest: \SCAN\ \cite{lake2018generalization} is one task specifically proposed to test for this property.
Previous work has achieved impressive empirical results using a group-equivariant neural network that naturally encodes a useful inductive bias for \SCAN\ \citep{gordon2019permutation}.
Inspired by this, we introduce a novel group-equivariant architecture that incorporates a group-invariant hard alignment mechanism.
We find that our network's structure allows it to develop stronger equivariance properties than existing group-equivariant approaches.
We additionally find that it outperforms previous group-equivariant networks empirically on the \SCAN\ task.
Our results suggest that integrating group-equivariance into a variety of neural architectures is a potentially fruitful avenue of research, and demonstrate the value of careful analysis of the theoretical properties of such architectures.
\newline
\newline
\vspace{1.5em}
\hspace{.5em}\includegraphics[width=1.25em,height=1.25em]{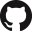}\hspace{.75em}\parbox{\dimexpr\linewidth-10\fboxsep-10\fboxrule}{\url{https://github.com/rycolab/equivariant-transduction}}
\vspace{-.5em}
\end{abstract}

\section{Introduction}
Humans painlessly process sentences they have never heard before.
This feat is possible because they can construct the meaning of a sentence by composing the meaning of its parts.
This phenomenon is known as \textbf{compositional generalization} in the computational linguistics literature \cite{lake2018generalization, hupkes2020compositionality} and it is what enables the understanding of an infinite number of novel sentences from only a finite set of words \citep{chomsky2009syntactic, montague1970universal}.
For example, without knowing what action the verb \emph{to blick} describes, we know that \emph{blick twice} indicates that this action should be performed two times.
It is natural that we would like for neural network models of language to be endowed with this ability as well---and, indeed, if they are to achieve human-like performance, it is likely to be necessary.

There have been multiple proposals for methods of assessing these abilities in neural models \cite{bahdanau2019closure, hupkes2020compositionality}.
One popular benchmark is the \SCAN\ task \citep{lake2018generalization} and its derivatives \cite{ruis2020benchmark}. 
\SCAN\ involves translating natural language instructions into a sequence of actions executable by an agent navigating in an environment.
Examples of these commands are shown in \cref{fig:scan}. 
In order to test generalization capabilities, there are various test--train splits which deliberately hold out specific words.
For example, the training set for the Add Jump split only contains \jumpin\ in the simple input--output pair (\jumpin, \jumpout), while the test set includes commands where it is combined with modifiers, such as \jumpin\ \twicein\ and \jumpin\ \leftin.
Since the model has seen the effect of these modifiers on other verbs in training, a model with the ability to perform compositional generalization should be able to apply them to \jumpin.\looseness=-1
\begin{figure}
    \centering
    \small
    \begin{tabular}{lcl}
             \walkin\ \rightin & $\rightarrow$ &\rightout\ \walkout\\
                 \walkin\ \andin\ \jumpin\ \twicein & $\rightarrow$ & \walkout\ \jumpout\ \jumpout\\
    \jumpin\ \leftin\ \afterin\ \walkin  & $\rightarrow$ & \walkout\  \leftout\ \jumpout\\
    \end{tabular}
    \caption{Examples of \SCAN\ inputs and outputs}
    \label{fig:scan}
\end{figure}

\SCAN\ is an example of a task that could benefit from a group-equivariant network.
When a network possesses the property of equivariance to a group $G$, acting on the input with an element of $G$ results in the output differing by the action of the same group element.
With an appropriate choice of group, this property can be used to imbue a network with the ability to generalize compositionally.
If we have a set of input words, and a corresponding set of their output words, and we act on these with a permutation group which swaps words within these sets, then swapping one word for another in the input will result in the corresponding output words being swapped in the output.
This effectively means that when the network has learned an example in training it can generalize to any input which can be reached from the training example by acting on it with a group element.
In the case of the \SCAN\ task, the action on a sentence amounts to replacing a word with another vocabulary word from the same lexical class (e.g., replacing a verb with another verb).
For example, if a $G$-equivariant network is trained on the \SCAN\ task using a group $G$ whose action amounts to swapping \SCAN\ verbs, then the probability of \runin\ being mapped to \runout\ will be identical to the probability of \jumpin\ being mapped to \jumpout\ -- thus even if only one verb is seen in training, it learns to apply observed patterns to the unseen verbs.
This approach was first applied to \SCAN\ by \newcite{gordon2019permutation}.

We present a novel group-equivariant architecture for the \SCAN\ task which has a notable theoretical advantage over similar existing approaches: the effective orbit of the model is larger, which results in more robust generalization to novel examples.
From one input command, our model can potentially generalize to an exponential number of unseen examples, where previous group-equivariant models could generalize only to a constant number.
We also demonstrate the empirical effect of this advantage, showing that our model outperforms that of \citeauthor{gordon2019permutation} across all splits of the \SCAN\ task.

Concretely, we incorporate group equivariance into the hard alignment string-to-string transduction model described by \citet{wu-etal-2018-hard}.
Hard alignment differs from the more common soft attention \citep{bahdanau2014neural} in that each output symbol is aligned with precisely one input symbol, rather than calculating a weighting over all input symbols.
Our model combines a group-invariant hard alignment mechanism with a group-equivariant transduction mechanism, which enables its improved generalization capabilities.
Our findings motivate further exploration of group-equivariant architectures, and suggest that careful consideration of their provable equivariance properties is worthwhile.

\section{Related Work} 
Many attempts have been made to quantify what it means for a model to exhibit compositional behaviour and
 to create models that can successfully generalize compositionally.
\citet{hupkes2020compositionality} provide a summary article on this topic, in which they describe different aspects of compositional generalization and formulate methods for assessing a model on each aspect.
In this work, we concern ourselves primarily with what \citeauthor{hupkes2020compositionality} term \textbf{systematicity}---the ability to understand an unseen combination of previously seen parts.

The \SCAN\ task which we focus on has been approached in many ways.
\citet{liuscan} achieve state-of-the-art results on the task, achieving 100.0 across all splits, using a memory-augmented model, trained using reinforcement learning, which learns how to identify the symbolic functions described by specific phrases within inputs.
\citet{lake2019compositional} approach the task using meta sequence-to-sequence learning.
Although this approach produces excellent results, it requires a bespoke meta-training approach for each generalization task.
\citet{russin2019compositional} achieve good results by treating syntax and semantics separately.
Their model separately calculates likely alignments between input and output words, and calculates how likely a word is to be produced conditioned on the input word with which it is aligned.
This approach is similar in concept to the hard alignment model that we use as the base for our $G$-equivariant architecture.

Equivariant networks built using group convolutions \citep{kondor2018generalization} have been used most in the field of computer vision, where they have been used to create Group Convolutional Neural Networks ($G$-CNNs) imbued with the property of invariance to transformations of an image such as rotation and reflection \citep{cohen2016group}.
Recently these techniques have been adopted for NLP tasks: \citet{gordon2019permutation} applied group-equivariant networks to the task of compositional generalization.
They used the building block of the group convolution to construct group-equivariant network components, including an LSTM, an attention mechanism and an embedding layer.
Their model is a group-equivariant analogue of a standard LSTM-based sequence-to-sequence model.
They evaluate it on the \SCAN\ task and achieve high accuracies on the splits assessing systematicity.
We build on their work by presenting an alternative group-equivariant architecture for this task, resulting in more robust generalization abilities and improved empirical performance.

\section{Group-Equivariant Networks}

\subsection{Group Theory}
Some understanding of the basics of Group Theory are helpful to understand group-equivariant networks.
This section will briefly outline the key concepts necessary for understanding our work.
\begin{defn}
A \textbf{group} $(G, \circ)$ is a set $G$ with a binary operation $\circ$ with the following properties:
\begin{enumerate}[label=(\roman*)]
    \item \textbf{Closure:} For any $g, h \in G$, $g\circ h \in G$.
    \item \textbf{Identity:} There exists an element $e$, which we call the identity, such that for all $g \in G$, $e\circ g = g \circ e = g$.
    \item \textbf{Inverse:} For any $g \in G$ there is an element $h \in G$ such that $g\circ h = h \circ g = e$. We call this element the inverse of $g$ and we may write $h = g^{-1}$.
    \item \textbf{Associativity:} for any $g_1, g_2, g_3 \in G$ we have $(g_1 \circ g_2)\circ g_3 = g_1 \circ (g_2 \circ g_3)$.
\end{enumerate}
\end{defn}
\begin{defn}
Let $G$ be a group and $X$ a set.
A \textbf{left group action} of $G$ on $X$ is a function $T: G \times X \rightarrow X$. 
As shorthand, we write $T(g,x) = g \circ x$. 
\end{defn}
\begin{defn}
If a group $G$ acts on a set $X$ through a defined action, then the \textbf{orbit} of an element $x \in X$ under $G$ (which we write $G \circ x$) is $G \circ x = \{ g\circ x \mid g \in G\}$.
\end{defn}
\begin{defn}
The \textbf{symmetric group} $S_n$ is the group of all permutations of a set of $n$ elements. 
The operation for this group is composition.
\end{defn}
If $G=S_n$ acts on a set $X=\{x_1,\dots,x_n\}$, then the canonical action of a group element permutes the items in the set. 
We will write elements of $S_n$ in cycle notation, where $g = (a_1 a_2 \dots a_m)$ indicates that $g\circ a_i = a_{i+1}$ for $i \in \{1, \dots, m-1\}$ with $g \circ a_m = a_1$ (and by extension $g \circ x_{a_i} = x_{g\circ a_i}$).
Indices that do not appear in the cycle are left unchanged.
$()$ indicates the identity permutation, which maps each item to itself.
An understanding of permutation groups will be key to our exposition, so we give an example of $S_3$. 
\begin{example}
$S_3$ is the symmetric group of permutations on 3 elements.
In full, $S_3 = \{(), (12), (13), (23), (123), (132)\}$.
It can act on any set with 3 elements.
For example, if $X=\{a, b, c\}$ then $(12) \circ a = b$, $(12)\circ b = a$ and $(12)\circ c = c$, since $(12)$ maps the first element of the set to the second element and vice versa, while leaving the third element unchanged. 
\end{example}
\begin{defn}
A group $G$ is \textbf{cyclic} if there is an element $g \in G$ such that all elements of $G$ take the form $g^i$, where $g^i = \underbrace{g\circ g \circ \cdots \circ g}_{i} $ for some integer $i$.
We say that $G$ is generated by $g$, and write $G = \langle g \rangle$.
\end{defn}
\begin{example}
One example of a cyclic group is a subgroup of $S_n$ generated by the permutation $g = (123\dots p)$ for $p \leq n$.
This permutation $g$ has the effect of shifting each of these $p$ elements by $1$, so $g^i$ has the effect of shifting each element by $i$, and $g^p$ is equal to the identity permutation.
We call this group a \textbf{cyclic shift group}.

To revisit the example of $S_3$, $G = \langle (123) \rangle = \{(), (123), (132)\}$ is an example of a cyclic shift group.
If this group acts on a set of 3 elements, then the identity element $()$, leaves the items in the set unchanged, while $(123)$ ``shifts'' each item along by one and $(132)$ by two.
\end{example}
Cyclic shift groups will be useful in our work.
If we have a set $X$ of $p$ objects that we wish to permute, the symmetric group of all possible permutations $S_p$ contains $p!$ elements.
However, in our work we will consider each object in isolation, so we do not to consider every permutation: we only need a group $G$ such that for a given $x_1 \in X$, for any element $x_2 \in X$ there is a permutation $g \in G$ such $g\circ x_1 = x_2$.
This is true of the cyclic shift group $G=\langle(123\dots p)\rangle$, which contains only $p$ elements.
By using cyclic shift groups in place of the symmetric groups, we are able to avoid unnecessarily slow calculations.
\begin{defn}\label{defn:equi}
Let $X$ and $Y$ be sets and $G$ be a group. 
Suppose the elements of $G$ act on $X$ and $Y$ with actions denoted $g \circ x$ and $g \circ y$, $x \in X$ and $y \in Y$.
A function $\alpha: X \rightarrow Y$ is \textbf{equivariant} with respect to $G$ (or $G$-equivariant) if and only if $\alpha(g \circ x) = g \circ \alpha(x)$ for all $x \in X$ and $g \in G$.
We say that $\alpha$ is \textbf{invariant} to $G$ if $\alpha(g \circ x) = \alpha(x)$ for any $x \in X$ and $g \in G$.
\end{defn}

\subsection{\SCAN: A Test of Compositionality}
The \SCAN\ task was proposed by \citet{lake2018generalization} to test a model's ability to generalize to unseen examples by composing known elements.
The task involves translating between an instruction in a limited form of English with input alphabet $\Sigma = \{\walkin, \jumpin, \ldots\}$ and an executable command with words taken from output alphabet $\Delta = \{\walkout, \jumpout, \ldots\}$.
As input the model receives a string $\xx \in \Sigma^*$, such as \jumpin\ \twicein, or \walkin\ \afterin\ \runin, and as output it should produce a string $\yy \in \Delta^*$, like \jumpout\ \jumpout, or \runout\ \walkout.

\SCAN\ contains several splits, each designed to test a different kind of generalization.
We focus on the splits that aim to evaluate the systematicity of the model -- that is to say how well it can understand the unseen combination of previously seen parts.
This is done by withholding certain combinations from the training set and then evaluating the model's performance on these unseen combinations in the test set.
For example, the training set of the Add Jump split only contains \jumpin\ in the basic input--output pair (\jumpin, \jumpout), while the test set contains inputs where it is used in combination with other modifiers.
Since these modifiers have been seen before with other verbs, this tests how well the model generalizes from what it has been exposed to in training.

\subsection{Group-Equivariance in Transduction}
We term a set of words whose function in \SCAN\ is the same and that are used in identical contexts as a \textbf{lexical class}.
We can see that, for example, any \SCAN\ command containing \walkin\ would equally be a valid instruction if \walkin\ were replaced with \jumpin, \lookin, or \runin.
We call this class the Verb lexical class.
We divide the input vocabulary into lexical classes $L^{\text{in}}_1, \dots, L^{\text{in}}_I \subseteq\Sigma$.
If we designate a lexical class to be our equivariant lexical class $L^{\text{in}}_{\text{equi}}$, we take the group $G=\langle (12\dots|L^{\text{in}}_{\text{equi}}|)\rangle$ to be the cyclic shift group of size $|L^{\text{in}}_{\text{equi}}|$ -- i.e., the group generated by a permutation that shifts each element along by one.
For example, if $L^{\text{in}}_{\text{equi}}=\{\walkin,\lookin,\runin,\jumpin\}$, then $G = \langle (1234)\rangle$, whose elements are the permutations $\{(), (1234), (13)(24), (1432)\}$.
The group acts on $L^{\text{in}}_{\text{equi}}$ by permuting its elements, so $(1234) \circ \walkin = \lookin$, for example.
The group also acts in the same way on the output lexical class $L^{\text{out}}_{\text{equi}} = \{\walkout,\lookout,\runout,\jumpout\}$, with $(1234) \circ \walkout = \lookout$.
Previous work \cite{gordon2019permutation} has defined the action of $g\in G$ on a sentence $\xx = (x_1, \dots, x_N)\in\Sigma^*$ as $g \circ \xx = (g\circ x_1, \dots, g\circ x_N)$.\footnote{If $x_n\in \Sigma\setminus L^{\text{in}}_{\text{equi}}$, $g\circ x_n=x_n$.}

We now explain why group equivariance is a useful property for a task such as \SCAN, and how it can lead to models that can generalize compositionally.
Consider the \SCAN\ task with input vocabulary $\Sigma$, output vocabulary $\Delta$ and lexical classes $L^{\text{in}}_{\text{equi}}\subseteq\Sigma$, $L^{\text{out}}_{\text{equi}}\subseteq\Delta$.
Let $\func: \Sigma^* \rightarrow \Delta^*$ be a transducer that constitutes a $G$-equivariant function for $G=\langle (12\dots|L^{\text{in}}_{\text{equi}}|)\rangle$, with its action defined on $\Sigma^*$ and $\Delta^*$ as above.
If we have an input--output pair $(\xx_0, \yy_0)\in\Sigma^*\times \Delta^*$ such that $\func(\xx_0) = \yy_0$, then for all input--output pairs $(\xx, \yy)\in\Sigma^*\times \Delta^*$ such that $(\xx,\yy) = (g\circ\xx_0,g\circ\yy_0)$ for some $g\in G$, we have $\func(\xx) = \func(g\circ\xx_0)=g\circ\func(\xx_0)=g\circ\yy_0=\yy$.
Plainly, this means that if our function successfully transduces a pair
$(\xx_0, \yy_0)$, it will also transduce all pairs its orbit $G\circ (\xx_0,\yy_0)=\{(g\circ\xx_0,g\circ\yy_0) \mid \forall g \in G\}$.
If we consider the case where $\func$ is obtained from the output of a neural network, this would mean that it could generalize compositionally to these examples even if the network had not been trained on them.

\subsection{Constructing $G$-Equivariant Networks}\label{subsec:layers}
We now describe how to build a neural network that is equivariant to a group $G$ using $G$-equivariant building blocks.
We describe three such building blocks that we will use in the construction of our equivariant hard-alignment model.
\paragraph{\circled{1} $G$-Convolution.}

The $G$-convolution is the extension of the standard convolution to an arbitrary finite group $G$ \citep{kondor2018generalization}.
Let $G$ be a finite group and $\vf: \dom(\vf) \rightarrow \R^K$ an input function.
Suppose we have $D$ learnable filter functions where we denote the $d^{\text{th}}$ filter function as $\vpsi^{(d)}:  \dom(\vf) \rightarrow \R^K$.
Then the $G$-convolution of $\vf$ with $\{\vpsi^{(d)}\}_{d=1}^D$ is a $|G| \times D$ matrix
where each entry is given by the following:
\begin{align}\label{eq:gconv}
    \gconv(&\vf; \vpsi)_{g, d} \\
    &= \sum_{h \in \dom(\vf)} \vf(h)\cdot \vpsi^{(d)}(g^{-1}\circ h) \nonumber
\end{align}
As shown in \citet[\textsection ~6.1]{cohen2016group}, $G$-convolutions are $G$-equivariant.
\paragraph{\circled{2} $G$-Embed.}
It is often desirable to embed input in a vector space -- but for a $G$-equivariant network we require this step too to be equivariant.
\newcite{gordon2019permutation} show that a $G$-equivariant embedding can be obtained as a special case of a $G$-convolution where the input function for input word $x$ is a one-hot encoding, %
$x:\{1, \dots, |\Sigma|\}\rightarrow\{0,1\}$.
Thus the learnable filter functions will be one-dimensional, and there is no longer a need for the sum, as it is zero at all but one value.
We call our set of $K$ filter functions $\{\omega^{(k)}\}_{k=1}^K$, with $\omega^{(k)} : \{1, \dots, |\Sigma|\} \rightarrow \R$.
Then the $G$-equivariant embedding of an input $x$ is a $|G| \times K$ matrix $e(x)$ with entries
\begin{align}
    e(x)_{g,k} &= \gembed(x;\boldsymbol{\omega})_{g,k}\nonumber\\
    &= \omega^{(k)}(g^{-1} \circ x)
\end{align}
\paragraph{\circled{3} $G$-Decode.}
The output $\vphi: G \rightarrow \R^D$ %
of some composition of equivariant layers based on $G$-convolutions will be a function over group elements, so in order to calculate logits over an output distribution of items $\Delta$, \citet{gordon2019permutation} use a decoding layer
\begin{equation}
    \gdecode(\vphi, \widetilde{y}; \boldsymbol{\rho}) = \sum_{h\in G} \vphi(h)\cdot\boldsymbol{\rho}(h^{-1}\circ\widetilde{y})
\end{equation}
where $\widetilde{y}\in\Delta$ is a candidate output value and $\boldsymbol{\rho}: \{1, \dots, |\Delta|\} \rightarrow \R^D$ %
is a learnable filter function.
This layer is not a form of group convolution, but instead is equivariant due to parameter-sharing \cite{ravanbakhsh17a}.

\section{An Equivariant Transducer}
We now describe the architecture of our equivariant hard alignment transducer.
Let $\Sigma$ be an input alphabet and $\Delta$ be an output alphabet.
Given an input string $\xx \in \Sigma^*$ we create a model to calculate a probability distribution $p(\yy \mid \xx)$ over $\yy \in \Delta^*$.
We consider what it means for such a model to be equivariant with the following theorem.
\begin{thm}
Let $p(\yy\mid\xx)$ be a probability distribution over $\Delta^*$.
If $p(g\circ\yy\mid g\circ\xx)=p(\yy\mid\xx)$ for any $g\in G, \xx\in\Sigma^*, \yy\in\Delta^*$, then the transducer $\func:\Sigma^*\rightarrow\Delta^*$ defined by $\func(\xx)=\argmax_{\yy\in\Delta^*}p(\yy\mid\xx)$ is equivariant.
\end{thm}
\begin{proof}
Given an $\xx$, let $\yy^*=\func(\xx)=\argmax_{\yy\in\Delta^*}p(\yy\mid\xx)$.
Now $\func(g\circ\xx) = \argmax_{\yy\in\Delta^*}p(\yy\mid g\circ\xx)$.
But for any $\yy\in\Delta^*$, $p(\yy\mid g\circ \xx)=p(g^{-1}\circ\yy\mid g^{-1}\circ g\circ \xx)=p(g^{-1}\circ\yy\mid\xx)$, so $\func(g\circ\xx) = \argmax_{\yy\in\Delta^*}p(g^{-1}\circ\yy\mid \xx)$.
We know that the probability distribution conditioned on $\xx$ is maximized by $\yy^*$, so we have $g^{-1}\circ \func(g\circ\xx) = \yy^*$ and thus $\func(g\circ\xx) = g\circ \yy^*$ and the transducer is equivariant.
\end{proof}
Moving forward we will use this equivalent definition to discuss the equivariance of our model.

\subsection{A Hard-Alignment Transducer}
As stated above, we aim to construct a distribution $p(\yy \mid \xx)$ where $\yy \in \Delta^*$ and $\xx \in \Sigma^*$.
The key idea behind a hard-alignment
model is that the two strings $\xx$ and $\yy$ are aligned according to a latent alignment $\Aa \in \calA(\xx, \yy)$. 
Hard alignment is contrasted with a soft alignment \cite{bahdanau2014neural, luong-etal-2015-effective}, which does not have the interpretation as a latent variable.
For $\xx$ of length $N$, $\yy$ of length $M$, each $\Aa$ is a vector in $\{1, \ldots, N\}^{M}$. 
If we have $a_m = n$, then the input word $x_n$ and output word $y_m$ are aligned. 
Thus, $\calA(\xx, \yy)$ consists of all non-monotonic alignments such that each word in $\yy$ aligns to exactly one word in $\xx$.
Because there is no annotation for alignments between the strings, we marginalize over alignments to compute the distribution:
\begin{align}\label{eq:likelihood}
        &p(\yy \mid \xx) = \sum_{\Aa \in \calA(\xx, \yy)}\,p(\yy, \Aa \mid \xx) \\
        &= \sum_{\Aa \in \calA(\xx, \yy)} \prod_{m=1}^{M} p(y_m \mid a_m, \yy_{<m}, \xx)\,p(a_m \mid \yy_{<m}, \xx) \nonumber
\end{align}
\citet{wu-etal-2018-hard} show that
\cref{eq:likelihood} may be rewritten as follows using the distributive property
\begin{equation}
    \prod_{m=1}^{M} \sum_{a_m=1}^{N} \underbrace{p(y_m \mid a_m, \yy_{<m}, \xx)}_{\mathit{translator}}\,\underbrace{p(a_m \mid \yy_{<m}, \xx)}_{\mathit{aligner}}\label{eq:trans_align} 
\end{equation}
which is more efficient to compute.
Specifically, this allows the distribution $p(\yy \mid \xx)$  to be computed $\bigo{M\cdot N}$ rather than $\bigo{|\calA(\xx, \yy)|}$, which is exponential in both $M$ and $N$. 

This formulation allows us to completely separate the alignment probability $p(a_m \mid \yy_{<m}, \xx)$, which calculates how likely an output word is to align with each input word, from the word translation probability $p(y_m \mid a_m, \yy_{<m}, \xx)$, which calculates how likely each output word is to be produced when aligned with a given input word. 
This can also be viewed as separately treating the syntax (alignment probability) and semantics (word translation probability) of the input.
In this sense, the model is similar to that of \citet{russin2019compositional}.

We also experimented with a variation on this formulation in which the sum over $a_m$ is replaced with taking the maximum, as well as an annealed variation.
This is described in \cref{sec:variant}.\looseness=-1

\subsection{An Equivariant Translator}
We now explain how the translator term in \cref{eq:trans_align} is defined in order to ensure group-equivariance. 
Because $y_m$ and $x_{a_m}$ are both individual words---as opposed to full sentences---the distribution $p(y_m \mid x_{a_m}, a_m)$ is a simple classifier.
We use a composition of the $G$-equivariant layers described in \cref{subsec:layers}, so that\looseness=-1
\begin{align*}\label{eq:dist}
    p(y_m \mid x_{a_m}, a_m) &= \frac{\exp\left(\gdecode(\vphi, y_m) \right)}{\sum_{y' \in \Delta} \exp\left(\gdecode(\vphi, y') \right)}
\end{align*}
where $\vphi = \gconv(e(x_{a_m}))$.
So, for each input word $x_{a_m}$, we first obtain a $G$-Embedding, then pass this through a $G$-Convolution, then use $G$-Decode to obtain logits over possible output words, which are then fed through a softmax to obtain probabilities.\footnote{Additional non-linearities are introduced by passing the output of each layer through a $\tanh$ function.} %

\subsection{An Invariant Aligner}\label{sec:align}
The alignment term is parameterized using a recurrent neural network.
Before encoding, we replace each word with a symbol indicating its lexical class.
Formally, if we have $I$ disjoint %
lexical classes $L^{\text{in}}_{1}, \ldots,L^{\text{in}}_{I}\subseteq\Sigma$ that cover $\Sigma$ and $J$ disjoint lexical classes $L^{\text{out}}_{1}, \ldots,L^{\text{out}}_{J}\subseteq\Delta$ that cover $\Delta$, we define functions $\ell_\Sigma: \Sigma \rightarrow \{1, \ldots, I\}$ and $\ell_\Delta: \Delta \rightarrow \{1, \ldots, J\}$ such that $\ell_\Sigma(w) = i$ iff $w \in L^{\text{in}}_{i}$, $\ell_\Delta(w) = j$ iff $w \in L^{\text{out}}_{j}$.
We overload these so that for $\xx \in \Sigma^*$ of length $N$, $\ell_\Sigma(\xx)=(\ell_\Sigma(x_1), \dots, \ell_\Sigma(x_{N}))$ and analogously $\ell_\Delta(\yy) = (\ell_\Delta(y_1), \dots, \ell_\Delta(y_{M}))$ for $\yy \in \Delta^*$ of length $M$.
Then, our model is actually dependent on $\ell_\Sigma(\xx)$ and $\ell_\Delta(\yy_{<m})$.
This has the effect of delexicalizing the input and equivalence-classing the words.
For example, if the verb and direction lexical classes are being considered, the input \walkin\ \leftin\ \afterin\ \runin\ will, in effect, be represented as \texttt{<verb>} \texttt{<direction>} \afterin\ \texttt{<verb>}. 
This substitution imbues the alignment model with some useful theoretical properties that will be discussed in \cref{subsec:equi_prop}.\looseness=-1

The delexicalized input sequence $\ell_\Sigma(\xx)$ of length $N$ is encoded both forwards and backwards using an LSTM to produce hidden states $\hleftenc_{n}, \hrightenc_{n} \in \mathbb{R}^{d_h}$ for $n\in\{1,\dots,N\}$.
These are then concatenated to obtain $\henc_{n} = \hleftenc_{n} \bigoplus  \hrightenc_{n} \in \mathbb{R}^{2d_h}$.
The output sequence $\ell_\Delta(\yy)$ of length $M$ is encoded forwards to produce $\hleftdec_m \in \mathbb{R}^{d_h}$ for $m\in\{1,\dots,M\}$.
Then, we have the following distribution
\begin{equation}\label{eq:align}
    p(a_m \mid \ell_\Delta(\yy_{<m}), \ell_\Sigma(\xx)) = \frac{\exp(e_{ma_m})}{\sum_{n=1}^{N} \exp(e_{mn})}
\end{equation}
where
\begin{equation}
   e_{mn} =  {\hleftdec_m}^{\top} \boldsymbol{T} \henc_{n}
\end{equation}
and $\boldsymbol{T} \in \mathbb{R}^{d_h \times 2d_h}$ is a learned matrix.
\subsection{Theoretical Results and Discussion}\label{subsec:equi_prop}
We know that each word-to-word translator $p(y_m \mid x_{a_m}, a_m)$ is $G$-equivariant.
We now explore in more detail the theoretical properties of the entire model.
We begin by considering the properties of the alignment model.
\begin{defn}
We say that a function $f:X\times Y\rightarrow\mathbb{R}$ is invariant to a group $G$ if for any $x\in X,y\in Y, g\in G$, we have $f(g\circ x, g\circ y) = f(x,y)$.
\end{defn}
\begin{thm}
The alignment model $p(a_m \mid \ell_\Delta(\yy_{<m}), \ell_\Sigma(\xx))$ defined in \cref{sec:align} is invariant to $G$.\looseness=-1
\end{thm}
\begin{proof}
Let us take a group element $g \in G$.
Then $\ell_\Sigma(g\circ \xx) = \ell_\Sigma (\xx)$.
We can see this because $g\circ \xx = (g\circ x_1, \dots, g\circ x_N)$ and $g$ only permutes each word \emph{within} its lexical class.
Words within the same lexical class are mapped to the same value by $\ell_\Sigma$.
A similar argument shows that $\ell_\Delta(g\circ \yy_{<m})=\ell_\Delta(\yy_{<m})$.
So $p(a_m \mid \ell_\Delta(g \circ\yy_{<m}), \ell_\Sigma(g\circ\xx)) = p(a_m \mid \ell_\Delta(\yy_{<m}), \ell_\Sigma(\xx))$ and the model is invariant.\looseness=-1
\end{proof}
We can now examine the properties of the transducer as a whole.
\begin{thm}
The model $p(\yy \mid \xx)$ defined in \cref{eq:trans_align} is $G$-equivariant.
\end{thm}
\begin{proof}
We take a group element $g\in G$.
Then
\begin{align}
    p&(g\circ\yy\mid g\circ \xx) \nonumber\\
    &= \prod_{m=1}^{M} \sum_{a_m=1}^{N} p(g\circ y_m \mid a_m, g\circ x_{a_m}) \nonumber\\
    &\quad\quad\quad p(a_m \mid \ell_\Delta(g\circ\yy_{<m}), \ell_\Sigma(g\circ \xx)) \nonumber
\end{align}
We have seen that the translator part of the model is equivariant to $G$, and the alignment part is invariant, so we can rewrite the above as 
\begin{equation}
    \prod_{m=1}^{M} \sum_{a_m=1}^{N} p(y_m \mid a_m, x_{a_m}) p(a_m \mid \ell_\Delta(\yy_{<m}), \ell_\Sigma(\xx))\nonumber
\end{equation}
which is equal to $p(\yy\mid \xx)$ and thus the model is $G$-equivariant.
\end{proof}
We would like to understand what these equivariance properties actually mean for our model's ability to generalize to unseen sentences.
To aid in the discussion of this question, we define the following.
\begin{defn}
Given a model
$p(\yy\mid\xx)$ that is equivariant to a group $G$ and a sentence pair $(\xx, \yy)\in\Sigma^*\times\Delta^*$, we define the \textbf{theoretical orbit} of $(\xx_0,\yy_0)$ under the model as $\Omega_T((\xx_0,\yy_0)) = \{(g\circ \xx_0, g\circ \yy_0)\mid g\in G\}$.
For any $(x_i,y_i)\in \Omega_T((\xx_0,\yy_0))$, $p(\yy_i\mid\xx_i) = p(\yy_0\mid\xx_0)$.
\end{defn}
If we maximize $p(\yy_0\mid\xx_0)$ for a training pair $(\xx_0,\yy_0)$, any benefits of this are shared by elements of its theoretical orbit, allowing the model to generalize to these pairs.
So the size of the theoretical orbit of a sentence pair is a way of quantifying how widely a model can generalize from that pair.

In the case of our model, with the equivariant lexical class corresponding to the 4 \SCAN\ verbs, any sentence pair containing one of these verbs has a theoretical orbit of size 4.
If $(\xx_0, \yy_0)$ is a sentence pair containing two different \SCAN\ verbs -- for example, (\walkin\ \afterin\ \runin, \runout\ \walkout) -- this means that its theoretical orbit contains 4 sentences out of the 16 sentences of this form that can be constructed.
In addition to the theoretical orbit, which is guaranteed by the properties of the model, we also consider the behaviour of a trained model in practice.
\begin{defn}
Given a model $p(\yy\mid\xx)$ that is equivariant to a group $G$ and a sentence pair $(\xx_0, \yy_0)\in\Sigma^*\times\Delta^*$, we define the \textbf{observed orbit} of $(\xx_0,\yy_0)$ under the model as $\Omega_O((\xx_0,\yy_0)) = \{(\xx, \yy)\in\Sigma^*\times\Delta^*\mid -\log p(\yy\mid\xx) =-\log  p(\yy_0\mid\xx_0)\}$.
Clearly $\Omega_T((\xx_0,\yy_0))\subseteq\Omega_O((\xx_0,\yy_0))$.
\end{defn}
\begin{figure*}
    \centering
    \BeginAccSupp{Alt={A chart showing bubbles representing orbits of sentence pairs over 100 epochs of training. It shows that in the early epochs, there are 4 distinct orbits of size 4, but by epoch 50 there is 1 orbit of size 16.},method=escape}
    \includegraphics[scale=0.35]{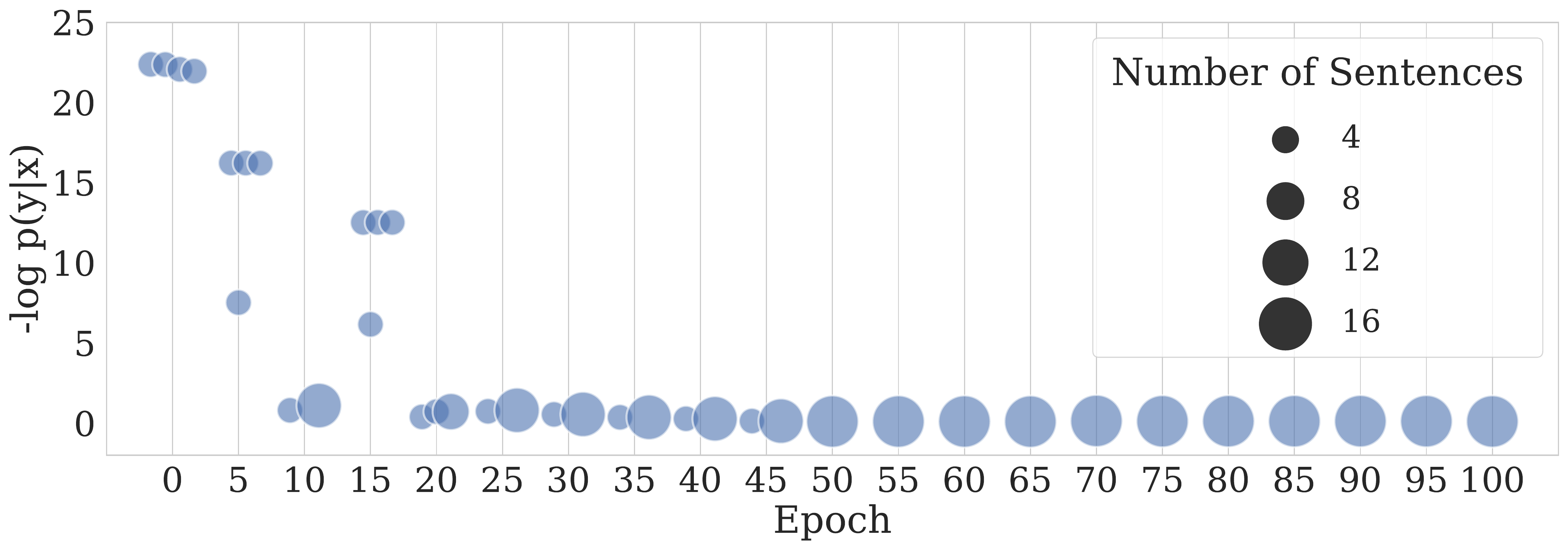}
    \EndAccSupp{}
    \caption{Observed orbits of sentence pairs $(\xx, \yy)$ during the first 100 epochs of training, with $\xx$ of the form \texttt{<verb$_1$> right thrice after <verb$_2$>} and $\yy$ of the form \textsc{<VERB$_2$> RTURN <VERB$_1$> RTURN <VERB$_1$> RTURN <VERB$_1$>}.}
    \label{fig:bubbleorbits}
\end{figure*}
To better understand how our model behaves, we examined the observed orbit of one of our best-performing models throughout training.
We selected sentence pairs from \SCAN\ containing 2 verbs, and generated all possible sentence pairs of the same form.
For each form, there are 16 such pairs.
While training our model, every five epochs we recorded the negative log-likelihood of each sentence pair as given by the model and counted how many of the 16 pairs had the same negative log-likelihood to find the size of the observed orbits.\footnote{The values were considered to be equal if they were within floating point error, as checked by \texttt{torch.isclose()}.}
We show the results for one sentence pair for the first 100 epochs of training in \cref{fig:bubbleorbits}. %
We can see that, initially, the 16 sentence pairs with the same form are in 4 distinct orbits of size 4, as predicted theoretically.
However, as the model trains, all 16 sentences come to have the same negative log-likelihood, and this continues for the remaining duration of training.

To better understand why the observed orbit under our model is larger than predicted theoretically, we examine the effect of alignments in more depth.
We consider a gold alignment for a \SCAN\ sentence pair $(\xx,\yy)$ to be an alignment $\Aa^*$ such that words from corresponding lexical classes are aligned and each output word is aligned with an input word with the corresponding meaning (e.g., \leftout\ aligned with \leftin, \walkout\ aligned with \walkin).

Given $\xx\in \Sigma^N$, $\yy\in\Delta^M$ and an alignment $\Aa$ between them, consider a group $G^N$ with elements $g=(g_1, \dots, g_N)$ with $g_n\in G$. %
Define its action on $\xx$ and $\yy$ by $g\circ \xx = (g_1\circ x_1, \dots, g_N\circ x_N)$ and %
$g\circ_{\Aa}\yy=(g_{a_1}\circ y_1, \dots, g_{a_M}\circ y_M)$.
Note that the group $G^N$ depends on the length of the string, and its action on $\yy\in\Delta^M$ depends on the alignment.
\begin{thm}\label{thm:gn-equi}
Given $\xx\in\Sigma^N$, $\yy\in\Delta^M$ and a gold alignment $\Aa^*\in\calA(\xx,\yy)$ between the two, $p(g\circ_{\Aa^*}\yy\mid g\circ\xx,\Aa^*)=p(\yy\mid \xx,\Aa^*)$ for all $g\in G^N$, with actions defined as above.
\end{thm}
\begin{proof}
For any $g\in G^N$, $p(g\circ_{\Aa^*}\yy \mid g\circ\xx, \Aa^*) =\prod_{m=1}^{M} p(g_{a^*_m}\circ y_m \mid g_{a^*_m}\circ x_{a^*_m})=\prod_{m=1}^{M} p(y_m \mid x_{a^*_m})= p(\yy\mid\xx, \Aa^*)$.
\end{proof}
This means that for $\xx_0\in\Sigma^N$, $\yy_0\in\Delta^M$ and a gold alignment $\Aa^*$, the model is equivariant to $G^N$ and thus the theoretical orbit of a sentence pair $(\xx_0, \yy_0)$ under the model is $|G|^N$ rather than $|G|$ as it was without conditioning on the alignment.\footnote{Unlike previous examples, \cref{thm:gn-equi} only \textit{directly} implies the $G^N$-equivariance of the restricted transducer $\func(\xx\mid\Aa^*)=\argmax_{\yy\in\Delta^M}p(\yy\mid\xx,\Aa^*)$, since the action of $G^N$ is only defined for $\yy \in \Delta^M$. %
However, it does in fact imply the equivariance of the unrestricted transducer, since conditioning on the alignment means that
$p(\yy\mid\xx,\Aa^*) = 0$ for $\yy\in\Delta^Q$ if $Q\not=M$.} %
Although this may seem irrelevant since we do not condition on gold alignments, we find that our trained models do approximate this.
As training progresses, our model gradually becomes more confident about the correct alignment between the input and output.
Once the model has confidently learned the gold alignment $\Aa^*$ between an input--output pair $\xx\in\Sigma^*, \yy\in\Delta^*$, $p(a_m\mid \yy_{<m},\xx)$ becomes very close to $0$ for $a_m \not= a^*_m$ and very close to $1$ for $a_m=a^*_m$, meaning that $p(\yy\mid\xx)$ approaches $p(\yy\mid\xx,\Aa^*)=\prod^{M}_{m=1} p(y_m\mid x_{a^*_m})$.
This means that for an input pair $(\xx_0,\yy_0)$, once the model has a high degree of confidence about the gold alignment $\Aa^*$ and $p(\yy_0\mid\xx_0)$ approaches $p(\yy_0\mid\xx_0,\Aa^*)$, the observed orbit of $(\xx_0,\yy_0)$ will become this larger theoretical orbit.
This explains what is shown in \cref{fig:bubbleorbits}.
\section{Experiments and Results}
\begin{table*}[t]
    \centering
    \begin{tabular}{lllll}
    \toprule
        Model & Simple & Add Jump & Around Right & Length \\ \midrule
        \citet{russin2019compositional} & 100.0 & 91.0  & 28.9 & 15.2 \\
        \citet{liuscan} & 100.0 & 100.0 & 100.0 & 100.0\\
        \citet{gordon2019permutation} & 100.0 & 99.1 & 92.0 & 15.9 \\
        Equivariant Hard Alignment & 100.0 & 100.0 & 100.0 & 28.5\\ \bottomrule 
    \end{tabular}
    \caption{Accuracy achieved on \SCAN\ task, presented alongside results from state-of-the-art systems}
    \label{tab:results}
\end{table*}
Experiments were performed on the Simple, Add Jump, Around Right and Length splits of \SCAN.
In each case, models were trained on 90\% of the train set, by minimizing the negative log-likelihood of observed sentence pairs, with 10\% of the train set reserved to be used as a validation set.
Models were selected to have the lowest loss on the validation set, and then evaluated on the test set.

For the Add Jump split, the group used was the cyclic shift group of size 4 acting on the set of \SCAN\ verbs.
For the Around Right split, which withholds the combination \aroundin\ \rightin, the group used was the cyclic shift group of size 2 acting on the set of directions.
Hyperparameters were selected through a random hyperparameter search.
At test time, outputs were decoded using a beam search with 3 beams.\looseness=-1

Results are shown in \cref{tab:results}. 
We can see that our model outperforms \citeposs{gordon2019permutation} similar group-equivariant model on all splits, achieving 100\% accuracy on all but the Length split, demonstrating empirically the advantage of the wider generalization capabilities that we explained theoretically.
It still falls short of \citeposs{liuscan} state-of-the-art memory-augmented model on the Length split.
Since the equivariance of our model primarily targets systematicity, which is not the ability targeted by the Length split, it is not surprising that it does not perform at state-of-the-art levels in this split.
We note, however, that our model does substantially outperform both \citeposs{gordon2019permutation} group-equivariant model, and \citeposs{russin2019compositional} semantic-syntactic alignment model on this split, which are the two models most similar to our own.

\section{Limitations and Future Work}
Group-equivariant networks have not yet been widely adopted in NLP, and we feel that there is potential for more tasks to be identified that could benefit from the application of group equivariance.
A further line of investigation would be to develop and assess group-equivariant versions of a wider variety of architectures.
The model used by \citet{gordon2019permutation} is an equivariant version of a standard LSTM-based sequence-to-sequence model, and the model used in our work is based on the hard alignment model used by \citet{wu-etal-2018-hard} -- these are only two of many possible architectures.
Variants of other architectures may possess different equivariance properties from both of these models, and may perform well on different tasks.
Future work could incorporate group equivariance into more architectures and assess their theoretical properties and empirical performance.

A key limitation of group-equivariant networks is the need to understand and specify the group to which it is equivariant.
This makes it difficult to apply these networks to real world tasks rather than artificial datasets such an \SCAN.
Future work could investigate ways to identify from data which words should be in a lexical class, or ways to allow group-equivariant networks to deal with open vocabulary problems.
Some work has been done on learning input and output vocabulary alignments in the context of \SCAN\ \citep{akyurek-andreas-2021-lexicon}, so it is possible that this could be used to improve group-equivariant architectures by reducing or eliminating their reliance on lexical classes and groups known \textit{a priori}.\looseness=-1

\section{Conclusion}
In this work, we proposed and implemented a string-to-string transduction model which combines a group-invariant sum over hard alignments and a group-equivariant output word probability to create a model which is equivariant to the swapping of words in the same lexical class.
We applied this model to the \SCAN\ task, finding that it is successful in allowing the model to generalize compositionally. 
We show theoretically that our model's structure allows for wider generalization to novel sentences than existing group-equivariant approaches and demonstrate this empirically.
We suggest that this is strong motivation to explore group-equivariant variants of other architectures, and to investigate other tasks which may benefit from group-equivariant models, as well as suggesting that theoretical analysis of equivariance properties may be a useful tool in understanding performance differences.

\bibliography{anthology,custom}
\bibliographystyle{acl_natbib}

\newpage
\appendix
\onecolumn
\section{Model Variants}\label{sec:variant}
In addition to the sum-based model presented in the main body of the paper, we also experimented on some variations of the model, which we explain here.
\subsection{Max Model}
The max version of the model is given by
\begin{equation}
    p(\yy \mid \xx) = \prod_{m=1}^{M} \max_{1\leq a_m\leq N} p(y_m \mid a_m, \yy_{<m}, \xx)\,p(a_m \mid \yy_{<m}, \xx) \label{eq:max}
\end{equation}
where all terms are the same as for the sum model.
\subsection{Annealed Max Model}
This model proved difficult to train, so we also conducted experiments with an annealed max model, which is given by
\begin{equation}
    p(\yy \mid \xx) = \prod_{m=1}^{M} \sum_{a_m=1}^{N} \alpha^m_{a_m} p(y_m \mid a_m, \yy_{<m}, \xx)p(a_m \mid \yy_{<m}, \xx) \label{eq:anneal}
\end{equation}
where $\boldsymbol{\alpha^m}$ is given by
\begin{equation}
    \alpha^m_{a_m}=\frac{\exp\left(\frac{1}{\tau}p(y_m \mid a_m, \yy_{<m}, \xx)p(a_m \mid \yy_{<m}, \xx)\right)}{\sum_{a_m^\prime=1}^{N} \left(\frac{1}{\tau}p(y_m \mid a_m^\prime, \yy_{<m}, \xx)p(a_m^\prime \mid \yy_{<m}, \xx)\right)}
\end{equation}
where $0<\tau\leq 1$ is the temperature.
As $\tau$ approaches $0$, $\boldsymbol{\alpha^m}$ approaches a one-hot vector indicating the $\argmax$ of $p(y_m \mid a_m, \yy_{<m}, \xx)p(a_m \mid \yy_{<m}, \xx)$, and thus \cref{eq:anneal} becomes close to \cref{eq:max}.
During training, $\tau$ is gradually decreased, so that the fully trained model is, in effect, a max model.
The starting value of $\tau$, as well as the schedule on which it is decreased, are hyperparameters.
\subsection{Model Comparison}
In \cref{tab:variant_results} we show the best accuracy achieved on each split by each variant of our model.

\begin{table*}[h]
    \centering
    \begin{tabular}{lllll}
    \toprule
        Model & Simple & Add Jump & Around Right & Length \\ \midrule
        Sum & 100.0 & 100.0 & 100.0 & 28.5\\ 
        Max & 100.0 & 99.9 & 99.9 & 18.3\\ 
        Annealed Max & 99.9 & 99.9 & 99.9 & 19.7\\ \bottomrule 
    \end{tabular}
    \caption{Accuracy achieved on \SCAN\ task by each variant of our model}
    \label{tab:variant_results}
\end{table*}

\section{Reproducibility}
The model was optimized using Adam \citep{kingma2014adam} with a learning rate of 0.001.
We chose hyperparameters through a random search.
Each hyperparameter included in the search in described in \cref{tab:hyperparam_range} along with the range that was searched.
\begin{table}
    \centering
    \begin{tabular}{lll}
    \toprule
    Hyperparameter & Details & Range\\\midrule
    Dimension of $G$-Embedding & Size of $G$-Embedding layer & 5-256\\
    Number of Filters & Number of filters used in $G$-Convolution layer& 5-256 \\
    Embedding Dimension  & Dimension of embedding layer in alignment model & 5-256\\
    Hidden Size & Size of LSTM layers used in alignment model & 5-256\\
    Batch Size & - & 8 - 64\\\bottomrule
    \end{tabular}
    \caption{Hyperparameter ranges that were searched}
    \label{tab:hyperparam_range}
\end{table}
For each split, sets of hyperparameters were randomly sampled and models with those hyperparameters were trained and evaluated.
\Cref{tab:hyperparam_values} shows the value of the hyperparameters for the best-achieving model on each split.
\begin{table}
    \centering
        \begin{tabular}{llllll}
    \toprule
    Split & \begin{tabular}{@{}l@{}}Dimension of\\$G$-Embedding\end{tabular} & \begin{tabular}{@{}l@{}}Number of\\ filters\end{tabular} & \begin{tabular}{@{}l@{}}Embedding\\Dimension\end{tabular}&  \begin{tabular}{@{}l@{}}Hidden\\Size\end{tabular} & \begin{tabular}{@{}l@{}}Batch\\Size\end{tabular}\\\midrule
    Simple & 20 & 24 & 36 & 6 & 8 \\
    Simple & 6 & 13 & 67 & 13 & 8\\
    Add Jump & 122 & 7 & 223 & 67& 8 \\
    Around Right  & 100 & 7 & 122 & 36 & 8 \\
    Around Right & 20 & 20 & 9 & 55  & 8 \\
    Around Right & 45 & 182 & 100 & 9 & 8 \\
    Length & 45 & 24 & 11 & 149 & 32\\\bottomrule
    \end{tabular}
    \caption{Values of hyperparameters for the best-performing sum-model on each split tested. Where more than one set of hyperparameters is given for a split, both performed equally well.}
    \label{tab:hyperparam_values}
\end{table}
In \cref{tab:max_hyperparam_values} and \cref{tab:annealed_hyperparam_values} we also include the best-performing hyperparameters for the max and annealed variations of the model.
\begin{table*}[p]
    \centering
        \begin{tabular}{llllll}
    \toprule
    Split & \begin{tabular}{@{}l@{}}Dimension of\\$G$-Embedding\end{tabular} & \begin{tabular}{@{}l@{}}Number of\\ filters\end{tabular} & \begin{tabular}{@{}l@{}}Embedding\\Dimension\end{tabular}&  \begin{tabular}{@{}l@{}}Hidden\\Size\end{tabular} & \begin{tabular}{@{}l@{}}Batch\\Size\end{tabular}\\\midrule
    Simple & 16 & 24 & 182 & 36 & 8 \\
    Add Jump & 30 & 11 & 182 & 55  & 16 \\
    Around Right  & 30 & 223 & 122 & 11 & 8 \\
    Length & 36 & 55 & 30 & 4& 16\\\bottomrule
    \end{tabular}
    \caption{Values of hyperparameters for the best-performing max-model on each split tested.}
    \label{tab:max_hyperparam_values}
\end{table*}
\begin{table*}[p]
    \centering
        \begin{tabular}{llllll}
    \toprule
    Split & \begin{tabular}{@{}l@{}}Dimension of\\$G$-Embedding\end{tabular} & \begin{tabular}{@{}l@{}}Number of\\ filters\end{tabular} & \begin{tabular}{@{}l@{}}Embedding\\Dimension\end{tabular}&  \begin{tabular}{@{}l@{}}Hidden\\Size\end{tabular} &  \begin{tabular}{@{}l@{}}Batch\\Size\end{tabular}\\\midrule
    Simple & 182 & 122 & 223 & 11  & 8\\
    Add Jump & 122 & 7 & 223 & 67  & 8 \\
    Around Right & 55 & 100 & 45 & 30 & 8 \\
    Length & 13 & 9 & 11 & 16 & 16\\\bottomrule
    \end{tabular}
    \caption{Values of hyperparameters for the best-performing annealed max-model on each split tested.}
    \label{tab:annealed_hyperparam_values}
\end{table*}
\section{Low-Data Experiments}
To test our model's ability to learn from small amounts of training data, we trained and evaluated a model with the best-performing hyperparameters on the Simple split using the low-data splits of \SCAN.
These splits have training sets containing 1\%, 2\%, 4\%, 8\%, 16\%, 32\% and 64\% of the total examples in the Simple split.
For comparison, we repeated this for all variants of our model, as well as for the equivariant model used by \newcite{gordon2019permutation}.
The results are shown in \cref{tab:lowdata}.
\begin{table}
    \centering
        \begin{tabular}{lllll}
    \toprule
    \begin{tabular}{@{}l@{}}Training Data\\Percentage\end{tabular} & \begin{tabular}{@{}l@{}}Sum\\Model\end{tabular} & \begin{tabular}{@{}l@{}}Max\\Model \end{tabular}&\begin{tabular}{@{}l@{}}Annealed\\Model\end{tabular} & \citeauthor{gordon2019permutation}\\\midrule
    1 & 42.14 & 0.47 & 33.26 & \textbf{45.75}\\
    2 & 68.09 &	56.28 & 58.36 &\textbf{73.53}\\
    4 & \textbf{89.59} & 9.2	& 81.22 &	89.54\\
    8 & 94.67 & \textbf{99.23}	&	95.61 &	96.60\\
    16 & 98.88 & \textbf{99.48 }& 98.44 & 97.12\\
    32 & \textbf{99.86} & 99.48 & 97.41 &	97.47\\
    64 & 99.67 & \textbf{99.88} & 99.27 & 99.67\\\bottomrule
    \end{tabular}
    \caption{Accuracy obtained by each model in low-data conditions}
    \label{tab:lowdata}
\end{table}
We can see that in the lowest data conditions, our model doesn't generalize as well as \citeposs{gordon2019permutation} model.
For the sets containing 16\% or more or the data, our model performs better.
We theorize that this is due to difficulty learning a high-quality alignment in the lowest-data conditions.
These results also show the volatility of the Max variant of our model -- it is able to reach above 99\% accuracy with less data than any of the other models, but with any less data it struggles to learn much at all.
\section{\SCAN\ Lexicon}
For further context we provide lists of the full \SCAN\ lexicon for input and output in \cref{tab:scan_in_vocab} and \cref{tab:scan_out_vocab} respectively.
In \cref{tab:direction} we list the words in the Direction lexical class in the input and output lexicon.
In \cref{tab:verbs} we do the same for the Verbs lexical class.
All other input words were in single-item lexical classes.
\begin{center}
\begin{minipage}[t]{.33\textwidth}
\vspace{0pt}
\centering
    \begin{tabular}{l}\toprule
    Input Words\\\midrule
    \runin\\
    \walkin\\
    \lookin\\
    \jumpin\\
    \leftin\\
    \rightin\\
    \afterin\\
    \andin\\
    \texttt{turn}\\
    \aroundin\\
    \twicein\\
    \thricein\\
    \texttt{opposite}\\
    \texttt{around}\\\bottomrule
    \end{tabular}
    \captionof{table}{All input vocabulary for the \SCAN\ task\label{tab:scan_in_vocab}}
\end{minipage}\hfil
\begin{minipage}[t]{.33\textwidth}
\vspace{0pt}
\centering
    \begin{tabular}{l}\toprule
    Output Words\\\midrule
    \runout\\
    \walkout\\
    \lookout\\
    \jumpout\\
    \leftout\\
    \rightout\\\bottomrule
    \end{tabular}
    \captionof{table}{All output vocabulary for the \SCAN\ task\label{tab:scan_out_vocab}}
\end{minipage}\hfil
\begin{minipage}[t]{.33\textwidth}
\vspace{0pt}
\centering
    \begin{tabular}{ll}\toprule
        Input & Output \\\midrule
        \rightin & \rightout \\
        \leftin & \leftout \\ \bottomrule
    \end{tabular}
    \captionof{table}{Direction lexical class in input and output vocabulary\label{tab:direction}}
    \vspace{10pt}
    \begin{tabular}{ll}\toprule
        Input & Output \\\midrule
        \runin & \runout \\
        \walkin & \walkout \\ 
        \lookin & \lookout \\
        \jumpin & \jumpout \\\bottomrule
    \end{tabular}
    \captionof{table}{Verb lexical class in input and output vocabulary\label{tab:verbs}}
\end{minipage}
\end{center}

\end{document}